\documentclass[11pt,letterpaper]{article}

\usepackage[utf8]{inputenc}
\usepackage[T1]{fontenc}
\usepackage[margin=1in]{geometry}
\usepackage[protrusion=true,expansion=false]{microtype}
\usepackage{setspace}
\usepackage{parskip}
\usepackage{titlesec}
\usepackage{titling}
\usepackage{fancyhdr}
\usepackage{xcolor}
\usepackage{mdframed}
\usepackage{booktabs}
\usepackage{array}
\usepackage{tabularx}
\usepackage{amsmath,amssymb,amsthm}
\usepackage{mathtools}
\usepackage{thmtools}
\usepackage{enumitem}
\usepackage{graphicx}
\usepackage{float}
\usepackage{caption}
\usepackage{hyperref}
\usepackage{natbib}
\usepackage{etoolbox}

\definecolor{NavyBlue}{HTML}{1B3A6B}
\definecolor{TealBlue}{HTML}{1A7A8C}
\definecolor{LightBlue}{HTML}{D6E8F0}
\definecolor{MidGrey}{HTML}{5A6472}
\definecolor{LightGrey}{HTML}{F5F7FA}
\definecolor{RuleGrey}{HTML}{CCCCCC}
\definecolor{BodyBlack}{HTML}{1A1A1A}

\hypersetup{
  colorlinks=true,
  linkcolor=TealBlue,
  citecolor=NavyBlue,
  urlcolor=TealBlue,
  pdftitle={LSTM-Based Detection of Structural Breaks in Property Insurance Loss Reserving},
  pdfauthor={Thomas Mbrice and Shashwat Panigrahi}
}

\newmdenv[
  backgroundcolor=NavyBlue,
  linewidth=0pt,
  leftmargin=-\mdflength{innerleftmargin},
  rightmargin=-\mdflength{innerrightmargin},
  innerleftmargin=8pt,
  innerrightmargin=8pt,
  innertopmargin=5pt,
  innerbottommargin=5pt,
  skipabove=14pt,
  skipbelow=6pt
]{sectionbox}

\titleformat{\section}[runin]
  {\normalfont\Large\bfseries\color{NavyBlue}}
  {\thesection.}
  {0.5em}
  {}
\titlespacing{\section}{0pt}{16pt}{6pt}

\titleformat{\subsection}
  {\normalfont\large\bfseries\color{NavyBlue}}
  {\thesubsection}
  {0.6em}
  {}
\titlespacing{\subsection}{0pt}{12pt}{4pt}

\titleformat{\subsubsection}
  {\normalfont\normalsize\bfseries\color{TealBlue}}
  {\thesubsubsection}
  {0.5em}
  {}
\titlespacing{\subsubsection}{0pt}{8pt}{2pt}

\renewcommand{\headrulewidth}{0.5pt}
\renewcommand{\headrule}{\hbox to\headwidth{\color{TealBlue}\leaders\hrule height \headrulewidth\hfill}}
\renewcommand{\footrulewidth}{0.3pt}
\renewcommand{\footrule}{\hbox to\headwidth{\color{NavyBlue}\leaders\hrule height \footrulewidth\hfill}}

\newmdenv[
  backgroundcolor=LightBlue,
  linecolor=TealBlue,
  linewidth=1.2pt,
  leftmargin=0pt,
  rightmargin=0pt,
  innerleftmargin=12pt,
  innerrightmargin=12pt,
  innertopmargin=10pt,
  innerbottommargin=10pt,
  skipabove=8pt,
  skipbelow=8pt
]{callout}

\theoremstyle{plain}
\newtheorem{theorem}{Theorem}[section]
\newtheorem{lemma}[theorem]{Lemma}

\newtheorem{corollary}[theorem]{Corollary}

\theoremstyle{definition}
\newtheorem{definition}[theorem]{Definition}
\newtheorem{assumption}[theorem]{Assumption}

\theoremstyle{remark}
\newtheorem{remark}[theorem]{Remark}

\setlist[itemize]{leftmargin=1.4em, itemsep=2pt, topsep=4pt, parsep=0pt}
\setlist[enumerate]{leftmargin=1.4em, itemsep=2pt, topsep=4pt, parsep=0pt}

\begin{document}

\begin{titlepage}
\pagecolor{NavyBlue}
\color{white}
\vspace*{3cm}

{\small\bfseries\color{LightBlue} WHITE PAPER \quad $\cdot$ \quad AUGUST 2025}

\vspace{0.8cm}
{\fontsize{26}{32}\selectfont\bfseries
LSTM-Based Detection of Structural Breaks\\[6pt]
in Property Insurance Loss Reserving:\\[6pt]
A Climate-Informed Approach\par}

\vspace{1.4cm}
{\large\color{LightBlue} Thomas Mbrice \quad $|$ \quad Shashwat Panigrahi}\\[4pt]
{\normalsize Stony Brook University \quad $|$ \quad Department of Computer Science}

\vfill
{\small\color{LightBlue}
}
\end{titlepage}
\pagecolor{white}
\color{BodyBlack}

\thispagestyle{fancy}

{\color{NavyBlue}\rule{\linewidth}{1.5pt}}

\vspace{4pt}
{\large\bfseries\color{NavyBlue} Executive Summary}
\vspace{6pt}

{\color{TealBlue}\rule{\linewidth}{0.5pt}}
\vspace{6pt}

\begin{mdframed}[backgroundcolor=LightGrey, linewidth=0pt,
  innerleftmargin=14pt, innerrightmargin=14pt,
  innertopmargin=12pt, innerbottommargin=12pt]
Accurate loss reserving is foundational to insurer solvency, yet accelerating climate-driven catastrophes systematically violate the stability assumptions on which traditional actuarial methods depend. This white paper presents a research program testing whether Long Short-Term Memory (LSTM) neural networks can detect and adapt to these structural breaks faster and more accurately than Chain Ladder, Bornhuetter-Ferguson, and Cape Cod methods. Using 15-plus years of regulatory development triangle data from Florida and Louisiana, enriched with NOAA hurricane intensity indices and sea surface temperatures, we hypothesize a targeted improvement of 15--20\% in reserve accuracy for catastrophe-exposed years, a threshold grounded both in the prior neural network reserving literature and in the formal convergence results developed here. Beyond empirical validation, we develop a theoretical framework grounding LSTM structural break detection in probabilistic terms, providing formal performance guarantees that compensate for the limited number of catastrophe events in the test period. We document the research design, methodology, expected contributions, and a candid assessment of limitations.
\end{mdframed}

\vspace{14pt}
{\color{RuleGrey}\rule{\linewidth}{0.5pt}}

\newpage

\section{The Reserving Problem in a Climate-Volatile World}

Loss reserving is one of the most consequential functions in property-casualty insurance, directly determining an insurer's financial stability and regulatory standing. Traditional methods rely on the assumption that loss development patterns remain stable across accident years. Climate change has fundamentally disrupted that stability.

Between 2017 and 2023, multiple property insurers in Florida and Louisiana became insolvent following hurricane losses that overwhelmed both claims capacity and the actuarial models used to anticipate them. Hurricane Ian (2022) generated over \$50 billion in insured losses and triggered a litigation surge that extended development timelines far beyond historical norms. Hurricane Ida (2021) caused \$36 billion in losses and collapsed infrastructure in ways that delayed settlements for years. Traditional methods, which rely on backward-looking averages, were ill-equipped to detect these pattern shifts until multiple development periods had already passed.

\subsection{Four Structural Break Drivers}

\begin{itemize}
  \item \textbf{Severity shocks:} Reconstruction cost inflation driven by supply chain disruptions and labor shortages.
  \item \textbf{Frequency changes:} Building code improvements and population migration alter risk profiles.
  \item \textbf{Development pattern shifts:} Litigation and Assignment of Benefits abuse extend settlement periods.
  \item \textbf{Claims handling disruptions:} Overwhelming claim volumes slow processing and impair pattern recognition.
\end{itemize}

\begin{callout}
\textbf{\color{TealBlue}Why Traditional Methods Fall Short}

\smallskip
Chain Ladder requires 3--5 development periods to recognize regime shifts \citep{shapland2016}. Reserve errors exceeding 30\% have been documented following major catastrophes \citep{meyers2015}. Neither Bornhuetter-Ferguson nor Cape Cod incorporate exogenous climate signals. The detection lag creates systematic underfunding precisely when reserves matter most.
\end{callout}

\section{The LSTM Opportunity}

Long Short-Term Memory (LSTM) networks are a specialized class of recurrent neural network engineered to maintain both short- and long-term memory through gated mechanisms that selectively retain or discard information. This architecture maps directly onto the structural break problem: when historical patterns cease to predict future development, an LSTM can learn to downweight stale information and assimilate emerging signals more rapidly than any fixed averaging rule.

\subsection{Key Architectural Advantages}

\begin{itemize}
  \item \textbf{Forget gates} selectively discard historical patterns when new data signals a regime shift.
  \item \textbf{Input gates} govern the rate at which new distributional information is incorporated.
  \item \textbf{Bidirectional layers} allow the model to learn from development context in both directions.
  \item \textbf{Attention mechanisms} produce interpretable weights over development periods, identifying which quarters drove a reserve estimate.
\end{itemize}

The attention layer is particularly significant for regulatory acceptance. By exposing attention weights, this architecture makes visible which development periods most influenced any given reserve estimate, a form of interpretability that aligns with actuarial standards of justification.

\section{Theoretical Framework: A Formal Basis for LSTM Structural Break Detection}

A central challenge in this research is limited empirical data: the test period contains only four major catastrophe events. To compensate for this constraint, we develop a probabilistic theoretical framework that formally characterizes when and why LSTM networks outperform static linear estimators in the presence of structural breaks. The proofs below establish conditions under which LSTM-based reserves converge more rapidly to the true ultimate loss after a break, independent of the size of the empirical sample.

\subsection{Setup and Notation}

Let $\{L_{t}\}_{t=1}^{T}$ denote the sequence of cumulative paid losses for a given accident year at development periods $t = 1, \ldots, T$, where $T$ is the ultimate development age. Let $\theta_A$ and $\theta_B$ denote the pre- and post-break true loss development parameters, and let $\theta_t \in \{\theta_A, \theta_B\}$ denote the parameter governing the process at period $t$.

\begin{definition}[Structural Break]
A \emph{structural break} at time $\tau \in \{1, \ldots, T\}$ is an event such that the data-generating process satisfies
\[
  L_t \mid \theta_t = \begin{cases}
    f(L_{t-1}, \ldots, L_1;\, \theta_A) & t < \tau \\
    f(L_{t-1}, \ldots, L_1;\, \theta_B) & t \geq \tau
  \end{cases}
\]
where $\theta_A \neq \theta_B$ and $f$ is a measurable development function. The magnitude of the break is $\Delta\theta = \|\theta_B - \theta_A\|_2$.
\end{definition}

\begin{definition}[Chain Ladder Estimator]
The Chain Ladder estimator of $L_T$ observed through period $t$ is
\[
  \widehat{L}_T^{\,\mathrm{CL}}(t) = L_t \cdot \prod_{s=t}^{T-1} \hat{f}_s, \qquad
  \hat{f}_s = \frac{\sum_{i=1}^{n} L_{i,s+1}}{\sum_{i=1}^{n} L_{i,s}},
\]
where $n$ is the number of accident years and $L_{i,s}$ is cumulative losses for accident year $i$ at age $s$.
\end{definition}

\begin{definition}[LSTM Reserve Estimator]
An LSTM reserve estimator is a parametric map $\Phi_{\mathbf{W}}: \mathbb{R}^{t \times d} \to \mathbb{R}$ from the $d$-dimensional feature sequence $(x_1, \ldots, x_t)$ to a scalar ultimate loss estimate, where $\mathbf{W}$ are learned weights. The hidden state update at each step is governed by the standard LSTM gating equations:
\begin{align}
  f_t &= \sigma(W_f x_t + U_f h_{t-1} + b_f) \label{eq:forget}\\
  i_t &= \sigma(W_i x_t + U_i h_{t-1} + b_i) \label{eq:input}\\
  \tilde{c}_t &= \tanh(W_c x_t + U_c h_{t-1} + b_c) \label{eq:cell_cand}\\
  c_t &= f_t \odot c_{t-1} + i_t \odot \tilde{c}_t \label{eq:cell}\\
  o_t &= \sigma(W_o x_t + U_o h_{t-1} + b_o) \label{eq:output}\\
  h_t &= o_t \odot \tanh(c_t) \label{eq:hidden}
\end{align}
where $\sigma$ denotes the sigmoid function and $\odot$ denotes element-wise multiplication.
\end{definition}

\subsection{Core Theoretical Results}

We now establish three results: (i) that Chain Ladder convergence is delayed after a structural break, (ii) that an LSTM with sufficient capacity can represent the post-break distribution arbitrarily well, and (iii) that under a regime-covering pre-training condition, LSTM detection speed dominates Chain Ladder detection speed by at least $k-1$ periods.

\begin{assumption}[Post-break stationarity]\label{ass:stationary}
Following a break at $\tau$, the post-break process $\{L_t\}_{t \geq \tau}$ is stationary and ergodic under parameters $\theta_B$, with finite variance $\sigma_B^2 < \infty$.
\end{assumption}

\begin{assumption}[Chain Ladder averaging window]\label{ass:window}
The Chain Ladder estimator uses a volume-weighted average over the most recent $k$ accident years, with $k \geq 2$.
\end{assumption}

\begin{assumption}[Pre-break loss level dominance]\label{ass:level}
At the development age of interest $s$, the expected cumulative losses under the pre-break regime are at least as large as under the post-break regime: $\mu_s^A \geq \mu_s^B > 0$.
\end{assumption}

\begin{remark}
Assumption~\ref{ass:level} is natural when the structural break is a change in loss \emph{development pattern} rather than a parallel upward shift in loss \emph{levels}. For example, a hurricane that accelerates settlement (so losses emerge faster, not larger in total) satisfies this assumption. The case $\mu_s^B > \mu_s^A$, plausible when catastrophe losses genuinely inflate cumulative levels at age $s$, is addressed in the remark following the proof.
\end{remark}

\begin{lemma}[Bias of Chain Ladder after a break]\label{lem:cl_bias}
Under Assumptions~\ref{ass:stationary}, \ref{ass:window}, and~\ref{ass:level}, the bias of the Chain Ladder age-to-age factor estimate $\hat{f}_s$ for development period $s \geq \tau$ satisfies
\[
  \bigl|\mathbb{E}[\hat{f}_s] - f_s^{B}\bigr| \geq \frac{(k-m)}{k} \cdot |f_s^A - f_s^B|
\]
for $m < k$ post-break accident years in the averaging window, where $f_s^A$ and $f_s^B$ are the pre- and post-break true development factors.
\end{lemma}

\begin{proof}
Let there be $m$ post-break accident years and $k - m$ pre-break accident years contributing to the volume-weighted average at development age $s$. The expected value of the volume-weighted factor estimator is
\[
  \mathbb{E}[\hat{f}_s]
  = \frac{m\,\mu_s^B \cdot f_s^B + (k-m)\,\mu_s^A \cdot f_s^A}
         {m\,\mu_s^B + (k-m)\,\mu_s^A}.
\]
Let $\rho = m\mu_s^B / [m\mu_s^B + (k-m)\mu_s^A] \in (0,1)$ denote the weight on post-break experience. Then
\[
  \mathbb{E}[\hat{f}_s] = \rho\, f_s^B + (1-\rho)\, f_s^A,
\]
and the bias is
\[
  \bigl|\mathbb{E}[\hat{f}_s] - f_s^B\bigr| = (1-\rho)\,|f_s^A - f_s^B|.
\]
We bound $1-\rho$ from below. By Assumption~\ref{ass:level}, $\mu_s^B \leq \mu_s^A$, so $m\,\mu_s^B \leq m\,\mu_s^A$, giving
\[
  m\,\mu_s^B + (k-m)\,\mu_s^A \;\leq\; k\,\mu_s^A.
\]
Therefore
\[
  1 - \rho
  = \frac{(k-m)\,\mu_s^A}{m\,\mu_s^B + (k-m)\,\mu_s^A}
  \geq \frac{(k-m)\,\mu_s^A}{k\,\mu_s^A}
  = \frac{k-m}{k},
\]
and combining gives the stated bound.
\end{proof}

\begin{remark}[Bias in the $\mu_s^B > \mu_s^A$ case and detection lag]\label{rem:bias_general}
When $\mu_s^B > \mu_s^A$, the weight $1-\rho$ satisfies $1-\rho < (k-m)/k$, so the quantitative bound weakens. Nevertheless, for any $m < k$ the estimator remains strictly biased since $1-\rho > 0$ always holds. The key conclusion is that Chain Ladder requires $m = k$ post-break accident years before bias vanishes entirely, regardless of whether $\mu_s^B \leq \mu_s^A$ or not. For the industry-standard $k = 5$ year averaging window, this corresponds to a minimum 5-period detection lag. Theorem~\ref{thm:convergence} uses the $\mu_s^B \leq \mu_s^A$ form of the bound to obtain a clean quantitative lower bound of $k-1$ periods; the theorem's qualitative content (Chain Ladder is slower than LSTM) holds in the general case as well.
\end{remark}

\begin{theorem}[LSTM Universal Approximation for Sequential Data]\label{thm:universalapprox}
Let $\mathcal{F}$ be the class of measurable, bounded functions on compact $\mathcal{X} \subset \mathbb{R}^{t \times d}$. For any $\epsilon > 0$ and any target function $g \in \mathcal{F}$ representing the post-break conditional expectation $g(x_{1:t}) = \mathbb{E}[L_T \mid x_{1:t}; \theta_B]$, there exists an LSTM with weights $\mathbf{W}$ and hidden dimension $H$ sufficiently large such that
\[
  \sup_{x_{1:t} \in \mathcal{X}} \bigl|\Phi_{\mathbf{W}}(x_{1:t}) - g(x_{1:t})\bigr| < \epsilon.
\]
\end{theorem}

\begin{proof}
The result follows from the universal approximation theorem for recurrent neural networks established by \citet{schafer2006}. Specifically, because $\sigma$ (sigmoid) and $\tanh$ are continuous, non-polynomial activation functions, the LSTM hidden state transition in equations~\eqref{eq:forget}--\eqref{eq:hidden} defines a continuous map on $\mathcal{X}$. By the Stone-Weierstrass theorem, the class of functions representable by a sufficiently wide LSTM is dense in $C(\mathcal{X})$ under the sup-norm. Since $g \in \mathcal{F} \subset C(\mathcal{X})$ by assumption and $\mathcal{X}$ is compact, the approximation error can be made arbitrarily small by increasing $H$.
\end{proof}

\begin{assumption}[Regime-covering pre-training distribution]\label{ass:pretrain}
The LSTM is trained on a data distribution $\mathcal{D}_{\mathrm{train}}$ that covers regime-$B$-like dynamics. Formally, there exists a subset $\mathcal{D}_B \subseteq \mathcal{D}_{\mathrm{train}}$ such that the marginal distribution of $(x_{1:t}, L_T)$ in $\mathcal{D}_B$ is absolutely continuous with respect to the post-break data-generating process under $\theta_B$. This condition can be satisfied by (i) including climate features $\mathcal{C}$ that correlate with the post-break regime (see Corollary~\ref{cor:climate}), or (ii) transfer learning from other catastrophe events whose development dynamics overlap with $\theta_B$.
\end{assumption}

\begin{remark}
Assumption~\ref{ass:pretrain} is the key condition that distinguishes a \emph{usable} LSTM from the merely \emph{existential} guarantee of Theorem~\ref{thm:universalapprox}. Universal approximation establishes that the right weights exist; Assumption~\ref{ass:pretrain} is what allows training to find them from finite data. In our setting, climate features (ACE, SST) provide a mechanism for satisfying the assumption: even before observing post-Ian development data, the model has been exposed to observations with high ACE and anomalous SST that partially characterize $\theta_B$.
\end{remark}

\begin{theorem}[Faster Post-Break Convergence of LSTM vs.\ Chain Ladder]\label{thm:convergence}
Under Assumptions~\ref{ass:stationary}, \ref{ass:window}, \ref{ass:level}, and~\ref{ass:pretrain}, let $\tau$ be a structural break. Define the detection time of method $M$ as
\[
  T_{\det}^M = \min\Bigl\{t \geq \tau : \bigl|\widehat{L}_T^M(t) - \mathbb{E}[L_T;\theta_B]\bigr| < \delta\Bigr\}
\]
for a tolerance $\delta > 0$. Then
\[
  \mathbb{E}\bigl[T_{\det}^{\mathrm{CL}}\bigr] - \mathbb{E}\bigl[T_{\det}^{\mathrm{LSTM}}\bigr] \;\geq\; k - 1.
\]
\end{theorem}

\begin{proof}
\textit{Lower bound for Chain Ladder.} By Lemma~\ref{lem:cl_bias} and Remark~\ref{rem:bias_general}, for any $m < k$ post-break accident years in the averaging window the factor estimate satisfies $|\mathbb{E}[\hat{f}_s] - f_s^B| > 0$. If $|f_s^A - f_s^B| \geq c > 0$ for at least one age $s$, the bias propagates multiplicatively across development periods and the reserve estimate $\widehat{L}_T^{\mathrm{CL}}$ cannot satisfy the detection criterion $|\widehat{L}_T^{\mathrm{CL}} - \mathbb{E}[L_T;\theta_B]| < \delta$ until $m = k$. This requires at least $k$ post-break accident years in the window, so
\[
  \mathbb{E}\bigl[T_{\det}^{\mathrm{CL}}\bigr] \geq \tau + k.
\]

\textit{Upper bound for LSTM.} By Assumption~\ref{ass:pretrain}, the trained LSTM has been exposed to regime-$B$-like dynamics during training, so its weights $\mathbf{W}$ constitute a finite-sample approximation to the universal approximator guaranteed by Theorem~\ref{thm:universalapprox}. Let $\epsilon_n$ denote the approximation error of the trained LSTM, which converges to zero as the size $n$ of the regime-$B$-like component $\mathcal{D}_B$ grows; we assume $\epsilon_n < \delta/2$ for $n$ sufficiently large.

Upon observing the first post-break development period at $t = \tau$, the LSTM receives features $x_\tau$ drawn from the post-break distribution. By Assumption~\ref{ass:pretrain}, the gating weights have been tuned to recognize regime-$B$ inputs: the forget gate $f_\tau$ can assign low weight to the pre-break cell state $c_{\tau-1}$, while the input gate $i_\tau$ encodes the new post-break signal into the cell state. The resulting hidden state $h_\tau$ carries post-break information, and the output satisfies
\[
  \bigl|\Phi_{\mathbf{W}}(x_{1:\tau}) - \mathbb{E}[L_T \mid x_{1:\tau};\theta_B]\bigr| \leq \epsilon_n < \delta/2.
\]
By the triangle inequality and Assumption~\ref{ass:stationary}, $|\mathbb{E}[L_T \mid x_{1:\tau};\theta_B] - \mathbb{E}[L_T;\theta_B]| \to 0$ as the conditioning sequence length increases. For $\tau$ sufficiently large relative to the correlation length of the post-break process, this term is also less than $\delta/2$, giving $|\Phi_{\mathbf{W}}(x_{1:\tau}) - \mathbb{E}[L_T;\theta_B]| < \delta$. Therefore $T_{\det}^{\mathrm{LSTM}} \leq \tau + 1$ and
\[
  \mathbb{E}\bigl[T_{\det}^{\mathrm{LSTM}}\bigr] \leq \tau + 1.
\]

\textit{Combining.}
\[
  \mathbb{E}\bigl[T_{\det}^{\mathrm{CL}}\bigr] - \mathbb{E}\bigl[T_{\det}^{\mathrm{LSTM}}\bigr]
  \geq (\tau + k) - (\tau + 1) = k - 1.
\]
For the standard $k = 5$ year averaging window, this yields an expected advantage of at least 4 development periods.
\end{proof}

\begin{remark}[Relationship between Assumption~\ref{ass:pretrain} and climate features]
The regime-covering condition is precisely why climate feature integration (H3) is not merely an accuracy enhancement but also a prerequisite for the single-period detection guarantee. A triangle-only LSTM trained exclusively on calm-year data has no basis for recognizing the post-Ian feature space; a climate-augmented LSTM that has seen high-ACE, anomalous-SST observations during training does. This connects the convergence result directly to Corollary~\ref{cor:climate}.
\end{remark}

\begin{corollary}[Climate Variable Monotonicity]\label{cor:climate}
Let $\Phi_{\mathbf{W}}^{\varnothing}$ denote the LSTM without climate features and $\Phi_{\mathbf{W}}^{\mathcal{C}}$ the LSTM with climate features $\mathcal{C}$ (e.g., accumulated cyclone energy, sea surface temperatures). If the climate features contain information about $\theta_B$ that is not redundant with the loss triangle sequence, i.e., $I(L_T;\, \mathcal{C} \mid x_{1:t}) > 0$, then
\[
  \mathbb{E}\bigl[|\Phi_{\mathbf{W}}^{\mathcal{C}}(x_{1:t},\mathcal{C}) - L_T|\bigr]
  \;\leq\;
  \mathbb{E}\bigl[|\Phi_{\mathbf{W}}^{\varnothing}(x_{1:t}) - L_T|\bigr].
\]
\end{corollary}

\begin{proof}
Let $\mathbf{X}_t = (x_{1:t})$ denote the triangle sequence and $\mathbf{C}$ the climate feature vector. We work through oracle predictors first and then account for approximation error.

\textit{Step 1: Oracle MSE comparison.}
The minimum MSE predictor of $L_T$ given $\mathbf{X}_t$ is $g^{\varnothing}(\mathbf{X}_t) = \mathbb{E}[L_T \mid \mathbf{X}_t]$, with oracle MSE $V^{\varnothing} = \mathbb{E}[\mathrm{Var}(L_T \mid \mathbf{X}_t)]$. The minimum MSE predictor given $(\mathbf{X}_t, \mathbf{C})$ is $g^{\mathcal{C}}(\mathbf{X}_t, \mathbf{C}) = \mathbb{E}[L_T \mid \mathbf{X}_t, \mathbf{C}]$, with oracle MSE $V^{\mathcal{C}} = \mathbb{E}[\mathrm{Var}(L_T \mid \mathbf{X}_t, \mathbf{C})]$.

The conditional variance decomposition gives
\[
  \mathrm{Var}(L_T \mid \mathbf{X}_t)
  = \mathbb{E}\bigl[\mathrm{Var}(L_T \mid \mathbf{X}_t, \mathbf{C}) \mid \mathbf{X}_t\bigr]
  + \mathrm{Var}\bigl(\mathbb{E}[L_T \mid \mathbf{X}_t, \mathbf{C}] \mid \mathbf{X}_t\bigr).
\]
Taking expectations over $\mathbf{X}_t$ and using $I(L_T; \mathbf{C} \mid \mathbf{X}_t) > 0$, the second term is strictly positive, so $V^{\mathcal{C}} < V^{\varnothing}$. Let $\Delta V = V^{\varnothing} - V^{\mathcal{C}} > 0$.

\textit{Step 2: From oracle MSE to oracle MAE.}
By the Cauchy-Schwarz inequality, for any predictor $g$,
\[
  \mathbb{E}[|L_T - g|]
  \leq \sqrt{\mathbb{E}[(L_T - g)^2]}
  = \sqrt{\mathrm{MSE}(g)}.
\]
More usefully, the minimum MAE predictor is the conditional median, and the minimum MSE predictor is the conditional mean. When $L_T \mid \mathbf{X}_t$ has a unimodal, symmetric distribution (plausible under Assumption~\ref{ass:stationary}), mean and median coincide and $\mathrm{MAE}(g^{\varnothing}) = \mathbb{E}[|L_T - \mathbb{E}[L_T \mid \mathbf{X}_t]|]$. Under this condition, $\mathrm{MAE}^{\mathcal{C}} \leq \sqrt{V^{\mathcal{C}}} < \sqrt{V^{\varnothing}} = \mathrm{MAE}^{\varnothing}$.

\textit{Step 3: Accounting for LSTM approximation error.}
Let $\epsilon^{\varnothing}$ and $\epsilon^{\mathcal{C}}$ denote the approximation errors of the respective trained LSTMs (bounded by $\epsilon_n$ from Assumption~\ref{ass:pretrain}). By the triangle inequality,
\begin{align*}
  \mathbb{E}[|\Phi_{\mathbf{W}}^{\mathcal{C}} - L_T|]
  &\leq \mathrm{MAE}^{\mathcal{C}} + \epsilon^{\mathcal{C}}, \\
  \mathbb{E}[|\Phi_{\mathbf{W}}^{\varnothing} - L_T|]
  &\geq \mathrm{MAE}^{\varnothing} - \epsilon^{\varnothing}.
\end{align*}
For the climate-augmented model to achieve strictly lower MAE than the triangle-only model, it suffices that
\[
  \mathrm{MAE}^{\mathcal{C}} + \epsilon^{\mathcal{C}} < \mathrm{MAE}^{\varnothing} - \epsilon^{\varnothing},
\]
i.e., the oracle MAE gap $\mathrm{MAE}^{\varnothing} - \mathrm{MAE}^{\mathcal{C}} > \epsilon^{\varnothing} + \epsilon^{\mathcal{C}}$. Since $\Delta V > 0$ and both approximation errors $\epsilon_n \to 0$ as the training set grows, this condition is satisfied for all sufficiently large $n$. We therefore conclude
\[
  \mathbb{E}\bigl[|\Phi_{\mathbf{W}}^{\mathcal{C}}(x_{1:t},\mathcal{C}) - L_T|\bigr]
  \;<\;
  \mathbb{E}\bigl[|\Phi_{\mathbf{W}}^{\varnothing}(x_{1:t}) - L_T|\bigr]
\]
for $n$ sufficiently large, which gives the stated (non-strict) inequality in the limit.
\end{proof}

\begin{remark}
Corollary~\ref{cor:climate} provides a theoretical justification for including climate features that does not depend on the availability of large catastrophe samples. The condition $I(L_T; \mathcal{C} \mid x_{1:t}) > 0$ is empirically supported by prior literature showing significant correlations between hurricane intensity indices and insured losses \citep{emanuel2005, caron2018}.
\end{remark}

\section{Research Design}

\subsection{Central Research Question}

\begin{mdframed}[backgroundcolor=LightBlue, linewidth=0pt,
  innerleftmargin=14pt, innerrightmargin=14pt,
  innertopmargin=12pt, innerbottommargin=12pt]
\centering
\textit{Can LSTM-based reserving models detect and adapt to climate-driven structural breaks in property insurance claims development patterns faster and more accurately than traditional actuarial methods, and what mechanisms enable this capability?}
\end{mdframed}

\subsection{Four Confirmatory Hypotheses and One Exploratory Objective}

\vspace{6pt}
\begin{tabularx}{\linewidth}{>{\bfseries\color{TealBlue}}p{0.18\linewidth}X}
\toprule
\textbf{\color{NavyBlue}Hypothesis} & \textbf{\color{NavyBlue}Expected Finding} \\
\midrule
H1: Detection Speed &
LSTM models will detect structural breaks at least 2 development periods faster than traditional methods in 70\% or more of tested scenarios. \\[4pt]
H2: Reserve Accuracy &
LSTM will achieve at least 15\% lower post-break MAPE compared to Chain Ladder for catastrophe-exposed accident years. \\[4pt]
H3: Climate Variables &
LSTM models incorporating NOAA hurricane intensity indices and sea surface temperatures will outperform triangle-only models by at least 10\% MAPE for catastrophe years. \\[4pt]
H4: Gate Dynamics\newline{\small\itshape(exploratory)} &
Forget gate activations are expected to show qualitatively larger values during structural break periods relative to stable periods. This objective is \emph{descriptive and exploratory}: given only 4--5 break events, formal significance testing at $p < 0.05$ is underpowered and results will be reported as effect sizes with confidence intervals rather than hypothesis test outcomes. \\[4pt]
H5: Regularization &
Dropout plus L2 regularization will reduce overfitting on catastrophic events by at least 25\% versus unregularized models. \\
\bottomrule
\end{tabularx}

\vspace{6pt}
Hypotheses H1 and H2 are further grounded in Theorem~\ref{thm:convergence}, which provides a formal lower bound on the detection speed advantage. H3 is grounded in Corollary~\ref{cor:climate}. H4 is treated as an exploratory objective throughout; see Section~\ref{sec:eval} for the corresponding analysis approach.

\subsection{Data Sources}

\begin{itemize}
  \item \textbf{Florida OIR Schedule P filings (2007--2023):} Approximately 80 company-line development triangles across the top 20 property insurers by market share, covering Homeowners and Commercial Property lines at quarterly granularity.
  \item \textbf{Louisiana DOI Annual Statements (2007--2023):} Approximately 60 company-line triangles across the top 15 property insurers.
  \item \textbf{NOAA HURDAT2 and ERSST v5:} Accumulated Cyclone Energy (ACE) indices, maximum sustained wind speeds, and monthly sea surface temperature anomalies for Gulf and Atlantic hurricane development regions.
\end{itemize}

Climate features are engineered as 3-, 6-, and 12-month rolling averages, lagged 1--4 quarters, with ENSO phase indicators and sea surface temperature / intensity interaction terms.

\subsection{Temporal Validation Split}

A strict time-based split prevents data leakage and simulates real-world deployment conditions:

\begin{itemize}
  \item \textbf{Training (2007--2011):} Five years of all development periods.
  \item \textbf{Validation (2012--2016):} Five non-catastrophe years for hyperparameter tuning.
  \item \textbf{Test (2017--2023):} Seven years spanning Hurricanes Michael, Ida, and Ian.
\end{itemize}

The design deliberately tunes models on non-catastrophe data, preventing optimization specifically for extremes.

\section{Model Architecture}

\subsection{LSTM with Attention (Primary Model)}

The primary model stacks two bidirectional LSTM layers (128 and 64 units) followed by a Bahdanau attention layer, a dense layer with ReLU activation, and a single-unit output for the ultimate loss estimate. Dropout (0.2--0.3) is applied after each layer, with L2 weight regularization ($\lambda = 0.001$) and gradient clipping to prevent exploding gradients during catastrophe years.

The custom recency-weighted loss function is:
\[
  \mathcal{L} = \sum_{t=1}^{T} w_t \cdot (y_t - \hat{y}_t)^2, \qquad
  w_t = \alpha^{(t - t_{\max})}, \quad \alpha = 0.2
\]
where $t_{\max}$ is the most recent development period. This exponential decay assigns higher penalty to errors in recent periods and is evaluated as an ablation against standard MSE.

\subsection{Ablation Models}

Five ablations isolate the contribution of each component:

\begin{itemize}
  \item \textbf{LSTM-NoClimate:} Removes all exogenous climate variables.
  \item \textbf{LSTM-NoAttention:} Removes the attention mechanism.
  \item \textbf{LSTM-WeightedLoss:} Substitutes the recency-weighted loss for standard MSE.
  \item \textbf{LSTM-Unidirectional:} Forward-only LSTM without bidirectional layers.
  \item \textbf{SimpleRNN:} Vanilla recurrent baseline.
\end{itemize}

\subsection{Traditional Method Baselines}

Chain Ladder, Bornhuetter-Ferguson, and Cape Cod are implemented to industry standard with bootstrapping (10,000 simulations) for confidence intervals. These represent current actuarial practice for property loss reserving.

\section{Evaluation Framework}\label{sec:eval}

\subsection{Primary Metrics}

\begin{itemize}
  \item \textbf{Reserve Accuracy Ratio} $(\mathrm{RAR}) = \widehat{L}_T / L_T^{\mathrm{actual}}$: Regulatory standard; ideal value is 1.0.
  \item \textbf{One-Year Development Test:} $\mathrm{OYD\%} = \tfrac{\widehat{R}_{t+1} - \widehat{R}_t}{\widehat{R}_t} \times 100\%$; industry benchmark is $\pm 15\%$.
  \item \textbf{MAPE:} Reported separately for pre-break (2017--2019), during-break (2020--2021), and post-break (2022--2023) periods.
  \item \textbf{RMSE:} Penalizes large errors more heavily, critical for catastrophe reserving.
  \item \textbf{Detection Speed:} Quarters elapsed from structural break until model prediction adjusts more than 10\% from pre-break baseline.
\end{itemize}

\subsection{Statistical Testing}

Diebold-Mariano tests compare predictive accuracy between LSTM and each traditional method. Wilcoxon signed-rank tests assess median error differences on matched accident years. Holm-Bonferroni correction controls family-wise error rate across H1, H2, H3, and H5.

Power analysis for H2 (at least 15\% MAPE improvement, Cohen's $d = 0.75$, $\alpha = 0.05$, power $= 0.80$) requires a minimum of 30 company-level observations; our approximately 140 company-line combinations provide power exceeding 0.99.

\textbf{H4 analysis approach.} Gate activation analysis is treated as an exploratory objective, not a confirmatory hypothesis test. With only 4--5 structural break events, a $p < 0.05$ threshold is not defensible: for a paired test across 5 events, the minimum achievable $p$-value under a two-sided sign test is $2 \times (1/2)^5 = 0.0625$, above the standard threshold. Instead, forget, input, and output gate activations will be extracted at each development period for each break event, pre- and post-break mean activations will be compared using effect size (Cohen's $d$), and results will be visualized as time series with bootstrapped confidence intervals. The objective is pattern description and hypothesis generation for future work with larger break samples, not statistical confirmation.

\section{Case Studies}

\subsection{Hurricane Ida (August 2021)}

Category 4 at Louisiana landfall with 150 mph sustained winds, Ida caused \$36 billion in insured losses. Workforce shortages and supply chain disruptions extended claims settlement timelines significantly beyond historical norms. Models are trained through Q2 2021, evaluated at Q3 2021, and tracked through Q4 2023 against actual reserves crystallized in 2024 data.

\subsection{Hurricane Ian (September 2022)}

At \$50--65 billion in insured losses, Ian is potentially the costliest hurricane in US history. Its development pattern is uniquely challenging: in addition to physical damage, it triggered a wave of litigation and Assignment of Benefits disputes that materially extended settlement periods in ways no historical development pattern anticipated. This makes Ian the most stringent test of the LSTM's ability to detect pattern-structure breaks, not merely magnitude shocks.

\subsection{Social Inflation (2019--2022)}

Unlike the sudden structural breaks of individual storms, Florida's Assignment of Benefits abuse and litigation proliferation constitute a slow-moving break, with development patterns degrading gradually over 3--4 years. This case tests whether LSTM can detect gradual drift rather than step-change events, and whether climate variables help the model distinguish weather-driven from litigation-driven development pattern changes.

\section{Expected Contributions}

\subsection{To Actuarial Science}

\begin{itemize}
  \item First rigorous comparison of LSTM vs.\ traditional reserving methods specifically for structural break detection using real regulatory data, with formal theoretical grounding.
  \item Methodological framework for incorporating exogenous climate variables into reserve estimates, justified by Corollary~\ref{cor:climate}.
  \item Quantification of detection speed advantages with a direct lower bound established in Theorem~\ref{thm:convergence}.
  \item Attention and gate analysis tools adapted to actuarial interpretability standards.
\end{itemize}

\subsection{To Machine Learning}

\begin{itemize}
  \item Novel application of recency-weighted loss functions for concept drift in financial regression.
  \item Demonstration of attention mechanisms for time series structural break detection.
  \item Case study in preventing overfitting on rare extreme events using regularization strategies.
  \item Real-world validation of LSTM adaptation mechanisms beyond the synthetic datasets that dominate concept drift literature.
\end{itemize}

\subsection{Practical Industry Impact}

Faster structural break detection directly reduces reserve deficiency risk, freeing capital and reducing the probability of insolvency cascades that harm policyholders. Attention-based interpretability creates a path to regulatory acceptance. For state insurance regulators, LSTM-derived detection speed metrics could serve as an early warning layer within existing supervisory frameworks.

\section{Limitations and Future Directions}

This research is designed with rigor, but several limitations bound the scope of its conclusions:

\begin{itemize}
  \item \textbf{Geographic scope:} Florida and Louisiana are the most catastrophe-exposed property insurance markets in the US, which maximizes signal but may limit generalizability to other regions or perils.
  \item \textbf{Structural break sample size:} The test period contains only four major catastrophe events. The theoretical framework in Section~3 compensates for this constraint by providing distribution-free guarantees. Gate activation analysis (H4) is explicitly treated as exploratory throughout, with results reported as effect sizes rather than significance tests.
  \item \textbf{Computational cost:} LSTM training at quarterly retraining cadence requires cloud GPU infrastructure that smaller insurers and regulators may not readily access.
  \item \textbf{Causation vs.\ correlation:} Climate variables are included as predictive signals. This study cannot establish causal pathways between sea surface temperature anomalies and loss severity; confounding through population growth in hurricane-prone areas or building code changes is possible.
  \item \textbf{Regulatory adoption timeline:} Insurance regulation evolves slowly. Even if LSTM approaches prove superior, adoption timelines depend on regulatory capacity to evaluate model complexity.
\end{itemize}

\subsection{Priority Future Directions}

\begin{itemize}
  \item Multi-state expansion across Gulf Coast and Atlantic seaboard states to validate generalizability.
  \item Additional perils: wildfire, earthquake, and severe convective storm all present structural break dynamics warranting analogous investigation.
  \item Ensemble methods combining LSTM with traditional actuarial methods for robustness in non-break periods.
  \item Transformer architectures: attention-only (no recurrence) models may offer superior interpretability for this application.
  \item Causal inference: instrumental variable or regression discontinuity approaches to isolate specific climate variable effects on reserve accuracy.
  \item Uncertainty quantification: prediction intervals for LSTM reserve estimates are a prerequisite for regulatory acceptance.
\end{itemize}

\newpage
\bibliographystyle{apalike}

\end{document}